\documentclass[journal]{IEEEtran}
\usepackage[utf8]{inputenc}
\usepackage{graphicx} 
\usepackage{amsfonts}
\usepackage{amsmath,amsthm,bm}
\usepackage{amssymb}
\usepackage[colorinlistoftodos]{todonotes}
\usepackage{algorithm}
\usepackage{algpseudocode}
\usepackage{longtable}
\newcommand{\probP}{\text{I\kern-0.15em P}}
\usepackage{color}
\usepackage{mathtools}
\usepackage{multirow}
\usepackage{float}
\usepackage{mathrsfs}
\usepackage{forloop}
\usepackage{url}
\usepackage[caption=false,font=footnotesize]{subfig}
\usepackage{hyperref}
\usepackage{booktabs}
\usepackage{makecell}
\usepackage{enumitem}

\newtheorem{theorem}{Theorem}
\newtheorem{assumption}{Assumption}
\newtheorem{corollary}{Corollary}[theorem] 

\title{Non-Normalized Solutions of Generalized Nash Equilibria in Dynamic Games With Shared Constraints}

\author{Mark Pustilnik$^{1}$, Francesco Borrelli$^{1}$
\thanks{$^{1}$University of California at Berkeley, USA \quad {\tt\small \{pkmark, fborrelli\}@berkeley.edu}}%
}

\begin{document}

\maketitle

\begin{abstract}

In dynamic games with shared constraints, Generalized Nash Equilibria (GNE) are often computed using the normalized solution concept, which assumes identical Lagrange multipliers for shared constraints across all players. While widely used, this approach excludes other potentially valuable GNE. This paper presents a simple and novel method that can utilize the Mixed Complementarity Problem (MCP) formulation to compute non-normalized GNE. The proposed approach allows to expand the solution space of GNE. We also propose a systematic approach for selecting one GNE based on a predefined optimality criteria, enhancing practical flexibility. Numerical examples from past literature and a car racing example illustrate the method’s effectiveness, offering an alternative to traditional normalized solutions. 

\end{abstract}



\section{Introduction}
Computing a Nash Equilibrium (NE) in a Dynamic Game (DG) is a well-established concept. In DGs, players often share constraints, making each player's feasible strategy set dependent on others’ strategies. This extends the problem to a Generalized Nash Equilibrium Problem (GNEP). NE serves as a fundamental tool for analyzing competitive behavior across various domains, including economics \cite{arrow2024existence} and competitive sports like racing \cite{zhu2023sequential, spica2020real}.

In general, the existence or uniqueness of a GNE is not guaranteed.
We are interested in problems where GNEPs admit multiple solutions. The standard approach for solving GNEPs is to derive and combine the Karush-Kuhn-Tucker (KKT) conditions for each player’s optimization problem~\cite{dreves2011solution}. When shared constraints exist, the Lagrange multipliers associated with these constraints are typically assumed to be identical for all players. This concept, known as the ``normalized solution", was introduced by Rosen \cite{rosen1965existence}, who demonstrated its existence and uniqueness under specific conditions. This solution also arises from reformulating the DG as a Quasi Variational Inequality (QVI) problem. Harker \cite{harker1991generalized} and others \cite{facchinei2007generalized} further showed that this reformulation reduces the solution set to the normalized solution, provided it exists. The normalized solution is widely adopted in various applications, including racing \cite{zhu2023sequential, liu2023learning, spica2020real}, economics \cite{bacsar1998dynamic}, and traffic systems \cite{migot2019revisiting}, and has become the standard approach for solving DGs due to its computational, numerical, and conceptual advantages. Despite these benefits, the normalized solution may not always capture the complexity of real-world interactions, where other GNEs could better represent asymmetric or more intuitive behavior. However, computing non-normalized GNEs is challenging. Standard approaches to solving GNEPs, such as reformulating the KKT conditions as a QVI is inherently bias toward normalized solutions \cite{facchinei2007generalized}. Reformulating the problem as Mixed Complementarity Problem (MCP) is practically also biased toward the normalized solution by many numerical solvers even when no such constraint is imposed on the solution space. The solution by the iterative best response \cite{spica2020real,wang2021game} may result in a non-normalized solution but any solution achieved by this method is heavily dependent on the initial guess and convergence is not guaranteed. The work done by Nabetani et al \cite{nabetani2011parametrized} suggest a method to calculate Non-Normalized solution by reformulating the QVI in a way that de-normalizes the solution. The resulting method requires solving a large number of VI instances - Each instance may or may not correspond to a valid GNE, even if the VI solver converges. For linear generalized Nash equilibrium problems, Dreves \cite{dreves2017computing} developed a duality-based algorithm that computes the entire solution set, not only normalized equilibria, and proved finite termination.

This paper addresses the limitations of existing GNEP solution methods through three key contributions:

\begin{enumerate}
    \item \textbf{A novel MCP-based framework for non-normalized equilibria:} We extend the Mixed Complementarity Problem (MCP) formulation to systematically compute \textit{non-normalized} Generalized Nash Equilibria (GNE), thereby uncovering solutions that are typically excluded by standard formulations.

    \item \textbf{A  GNE selection mechanism:} We propose a systematic bi-level optimization approach for selecting an ``optimal'' equilibrium from the expanded set of non-normalized GNEs. Crucially, we formally show that the introduced scaling factors  act as well-defined \textit{tuning knobs} for navigating and parameterizing the entire GNE solution set.

    \item \textbf{Empirical validation:} We demonstrate both analytically and numerically how normalized solutions can lead to unrealistic or counterintuitive outcomes, and we show that the proposed framework not only overcomes these issues but also yields richer, more intuitive, and practically relevant equilibria.
\end{enumerate}

The paper is organized as follows. Section~\ref{sec:Problem} introduces the problem formulation and a motivating example. Section~\ref{sec:Normalized} reviews the classical normalized GNE approach and its limitations. Section~\ref{sec:Non_Normalized} presents the proposed MCP-based method for computing non-normalized GNEs and the theoretical results that establish the scaling matrices as formal tuning parameters. Section~\ref{sec:examples} provides analytical validation, while Section~\ref{sec:selection} introduces the bi-level optimization for equilibrium selection. Section~\ref{sec:results} presents simulation results on dynamic racing games, and Section~\ref{sec:conclusion} concludes the paper.


\section{Problem Formulation} \label{sec:Problem}
Given $M$ players, each player $i \in \{1,...,M\}$ controls the variables $x^i \in \mathbb{R}^{n_i}$. 
The vector $x \in \mathbb{R}^n$ is formed by concatenating all the players' decision variables:

\begin{equation} \label{x_vec}
    x := [(x^1)^T, ..., (x^M)^T]^T
\end{equation}
where $n := \sum^M_{i=1}n_i$. $x^{-i}$ represent the decision variables of all players except player $i$. To emphasize the decision variables of player $i$ within $x$ we write $(x^i,x^{-i})$ instead of $x$. The aim of player $i$ is to choose the $x^i$ which minimizes its own cost function $J_i(x^i,x^{-i})$ refer to such decision as ``strategy".
The feasible set of the $i$-th decision depends on the strategies of the other players:
\begin{equation} \label{game_def}
    \min_{x^i} {J_i(x^i,x^{-i})} \quad \text{subject to} \quad x^i \in \mathcal{X}_i(x^{-i}) \subseteq \mathbb{R}^{n_i}
\end{equation}
The feasible set of each player can be represented in the following form:
\begin{equation} \label{feasible_set}
    \begin{split}
    \mathcal{X}_i(x_{-i}) := \{x_i \in \mathbb{R}^{n_i} | h_i(x_i) &=   0 \\ ,g_i(x_i) &\leq 0,~ s(x_i,x_{-i}) \leq 0\}
    \end{split}
\end{equation}
where, $h_i: \mathbb{R}^{n_i} \to \mathbb{R}^{k_i}$ defines private equality constraints that depend only on player $i$'s strategy, $g_i:\mathbb{R}^{n_i} \rightarrow \mathbb{R}^{m_i}$ defines the private inequality constraints of player $i$ that depends only on player $i$'s strategy. $s:\mathbb{R}^{n} \rightarrow \mathbb{R}^{m_0}$ defines the shared constraints that depends on the strategies of all players and shared by all players. The solution set of problem (\ref{game_def}) for player $i$ is denoted by $\mathcal{S}_i(\bar{x}^{-i})$. The Generalized Nash Equilibrium Problem (GNEP) is the problem of finding vector $\bar{x}$ such that:
\begin{equation} \label{game_sol}
    \bar{x}^i \in \mathcal{S}_i(\bar{x}^{-i}) \quad \forall i \in \{1,...,M\}
\end{equation}
A solution to the GNEP is the GNE. A GNE is the set of strategies that characterized by the fact that none of the players can improve unilaterally its cost function by changing its strategy in a feasible direction. Furthermore, it is typical to assume the following assumptions (Rosen's setting \cite{rosen1965existence}):
\begin{assumption}
\label{ass1}
\begin{enumerate}[label=(\roman*)] \label{assm}
    \item set $\mathcal{X} \subset \mathbb{R}^n$ is convex and compact.
    \item Function $g_i: \mathbb{R}^n \rightarrow \mathbb{R}, i=1,...,m$ and $s: \mathbb{R}^n \rightarrow \mathbb{R}^{m_0}$ are convex and continuously differentiable.
    \item The cost function of every player $J_i(x^i,x^{-i})$ is continuously differentiable in $x \in \mathcal{X}$.
    \item The cost function of every player $J_i(x^i,x^{-i})$ is pseudo-convex in $x^i$ for every given $x^{-i}$.
\end{enumerate}
\end{assumption}

Define the Lagrangian function for each player as:
\begin{equation} \label{Lagrangian}
    \mathcal{L}_i := J_i(x^i,x^{-i})+ \mu_i^Th_i(x^i) + \lambda_i^T \cdot g_i(x^i) + \sigma_i^T \cdot s(x^i,x^{-i})
\end{equation}
where, $\mu_i \in \mathbb{R}^{k_i}$ are the Lagrange multipliers of the equality constraints of player $i$, $\lambda_i \in \mathbb{R}_+^{m_i}$ and $\sigma_i \in \mathbb{R}_+^{m_0}$ are the Lagrange multipliers of the private and shared inequality constraints of player $i$ respectively. For a point $x \in \mathcal{X}$ to be a GNE, the following KKT conditions have to be satisfied:
\begin{equation} \label{KKT}
\begin{split}
    & \nabla_{x^i}\mathcal{L}_i =0, \quad \forall i =1,...,M \\
    & 0 \leq \lambda_i \perp g_i \leq 0, \quad \forall i =1,...,M \\
    & 0 \leq \sigma_i \perp s \leq 0, \quad \forall i =1,...,M \\
    & \lambda_i \geq 0, \quad \forall i =1,...,M \\
    & \sigma_i \geq 0, \quad \forall i =1,...,M \\
    & h_i = 0, \quad \forall i =1,...,M \\
\end{split}
\end{equation}
The following theorem summarizes known results on computing GNE as a solution to a GNEP.
\begin{theorem} \cite{dreves2011solution, bueno2019optimality, facchinei2010generalized} \label{thm0} Consider the game~(\ref{game_def})-(\ref{feasible_set}), assume  Assumption~\ref{ass1}  and some suitable Constraint Qualification  (e.g., Slater Conditions) hold, then \begin{enumerate} 
\item If the tuple $(\bar{x}, \{\bar{\mu}_i\}_{i=1}^M, \{\bar{\lambda}_i\}_{i=1}^M,\{\bar{\sigma}_i\}_{i=1}^M)$ satisfies the KKT conditions (\ref{KKT}), then $\bar{x}$ is a solution of the GNEP. 
\item If $\bar{x}$ is a solution of the GNEP, then there exists a suitable vectors of multipliers $(\{\bar{\mu}_i\}_{i=1}^M, \{\bar{\lambda}_i\}_{i=1}^M,\{\bar{\sigma}_i\}_{i=1}^M)$ that satisfy the KKT conditions (\ref{KKT}) when $x=\bar{x}$ 
\item If $\bar{x}$ is a solution of the GNEP and LICQ holds, then there exists a \emph{unique} vectors of multipliers $(\{\bar{\mu}_i\}_{i=1}^M, \{\bar{\lambda}_i\}_{i=1}^M,\{\bar{\sigma}_i\}_{i=1}^M)$ that satisfy the KKT conditions (\ref{KKT}) when $x=\bar{x}$ \end{enumerate} \end{theorem}

\subsection{GNE Modeling of a Three Car Racing Problem} \label{exm_2cars}
To illustrate the formulation presented in Section \ref{sec:Problem}, we consider a one-step, finite-horizon, discrete-time, general-sum, open-loop dynamic game in a racing scenario involving three cars on a two-lane track. Each car is a player in the generalized Nash equilibrium problem (GNEP), which controls its velocity along a one-dimensional lane. Car 1 occupies the first lane, while Cars 2 and 3 share the second lane (Figure~\ref{fig:race}). $x_i$ and $v_i$ represent the final position and velocity of the car $i$, respectively. The game for player $i \in \{1,2,3\}$:

\begin{equation} \label{exm:setup}
\begin{alignedat}{3}
    &\min_{x_i, v_i} J_i(x_i,v_i,x^{-i}) \\
    &\text{s.t.} \\
    &x_i = x_i(0) + v_i\cdot \Delta t  \\
    &x_2 \leq x_3 \\
    &x_1(0)=0,x_2(0)=0.5,x_3(0)=0.75
\end{alignedat}
\end{equation}

where $\Delta t = 1[sec]$. The cost functions of the players are:
\begin{equation} \label{exm:cost}
\begin{alignedat}{3}
    &J_1 = -x_1 + x_2 + \frac{1}{2}v_1^2 \\
    &J_2 = -x_2 + x_1 + \frac{1}{2}v_2^2 \\
    &J_3 = -x_1 + x_2 + \frac{1}{2}v_3^2 
\end{alignedat}
\end{equation} 
The objective of Car 1 is to advance farther than Car 2 while minimizing control effort. Similarly, Car 2 aims to advance farther than Car 1 while minimizing its control effort. Car 3, however, seeks to assist Car 1 in winning the race while keeping its control effort minimal. A collision avoidance constraint between Cars 2 and 3 makes it possible for car 3 to affect car 2 position. Figure \ref{fig:race} illustrates the initial conditions of the race. Throughout this paper, we will use this example to provide insights into the proposed methods.
\begin{figure} 
\centering 
\includegraphics[scale=0.45]{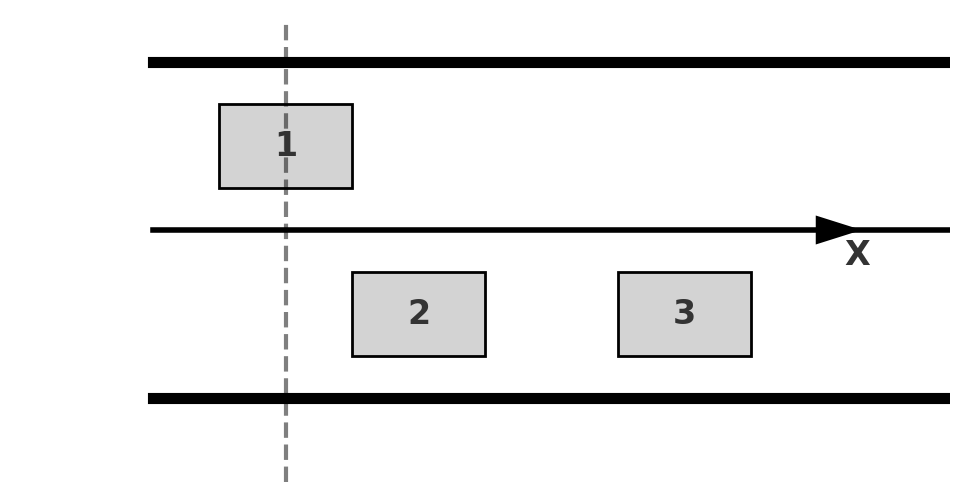} 
\caption{Illustration of the 1D racing example.} 
\label{fig:race} 
\end{figure}


\section{The Normalized Solution For Dynamic Games} \label{sec:Normalized}
Rosen introduced the concept of the ``Normalized" solution for the Generalized Nash Equilibrium Problem (GNEP) in \cite{rosen1965existence}. Under the previously stated assumptions, the normalized solution corresponds to a scenario where the Lagrange multipliers associated with the shared constraints are identical for all players:
\begin{equation} \label{shared_sigma}
    \sigma^1 = \dots = \sigma^M = \sigma.
\end{equation}
Rosen showed that a unique normalized solution exists under the specified conditions. Subsequent works, including \cite{harker1991generalized,facchinei2007generalized}, showed that reformulating the GNEP as a Quasi-Variational Inequality (QVI) problem inherently reduces the solution set to the normalized solution, see \cite{facchinei2007generalized} section 3. \textit{As a result, the normalized solution has become the de facto standard in solving GNEPs, to the extent that many papers do not explicitly mention this assumption}. One popular method for solving GNEPs is to reformulate the Karush-Kuhn-Tucker (KKT) conditions as a Mixed Complementarity Problem (MCP). An MCP is defined for a function $F(z): \mathbb{R}^d \to \mathbb{R}^d$, with lower and upper bounds $l \in \mathbb{R}^d$ and $u \in \mathbb{R}^d$, respectively. The goal is to find a vector $z^* \in \mathbb{R}^d$ such that the following conditions hold for each element $j \in \{1, \dots, d\}$:
\begin{equation} \label{MCP}
\begin{split}
    z^*_j = l_j &\iff F_j(z^*) \geq 0, \\
    l_j < z^*_j < u_j &\iff F_j(z^*) = 0, \\
    z^*_j = u_j &\iff F_j(z^*) \leq 0.
\end{split}
\end{equation}

The KKT conditions of the GNEP can be reformulated as an MCP by defining the following $F$, $l$, and $u$:
\begin{equation}
\begin{split}
F(z) &= 
[
\nabla^T_{x^1}\mathcal{L}_1, \cdots, \nabla^T_{x^M}\mathcal{L}_M ,h^T_1,\cdots,h^T_M, \cdots \\ &\qquad \qquad \qquad \cdots,-g^T_1,\cdots,-g^T_M,-s^T,\cdots,-s^T
]^T\\
l &= 
[-\infty,\cdots,-\infty,-\infty,\cdots,-\infty,0,\cdots,0,0,\cdots,0
]^T \\
u &= 
[\infty, \cdots, \infty,\infty ,\cdots ,\infty ,\infty ,\cdots ,\infty ,\infty ,\cdots ,\infty]^T
\label{eq:full_MCP}
\end{split}
\end{equation}
where $\bar{z} = [x^1, \dots, x^M, \mu_1, \dots, \mu_M, \lambda_1, \dots, \lambda_M, \sigma_1, \dots, \sigma_M]$ represents the vector of all decision variables and Lagrange multipliers, and the shared constraint $s$ is duplicated $N$ times. Solving the MCP above yields the normalized solution, as shown in \cite{facchinei2007generalized}. To simplify the formulation, the normalized solution assumption reduces the MCP by enforcing shared multipliers resulting in:
\begin{equation}
\begin{split}
F(z) &= 
[
\nabla^T_{x^1}\mathcal{L}_1, \cdots, \nabla^T_{x^M}\mathcal{L}_M ,h^T_1,\cdots,h^T_M, \cdots \\ &\qquad \qquad\qquad \qquad \cdots,-g^T_1,\cdots,-g^T_M,-s^T
]^T\\
l &= 
[-\infty,\cdots,-\infty,-\infty,\cdots,-\infty,0,\cdots,0,0]^T \\
u &= 
[\infty, \cdots, \infty,\infty ,\cdots ,\infty ,\infty ,\cdots ,\infty ,\infty ]^T
\label{eq:normalized_MCP}
\end{split}
\end{equation}
where $z = [x^1, \dots, x^M, \mu_1, \dots, \mu_M, \lambda_1, \dots, \lambda_M, \sigma]$.

Numerical solvers, such as PATH \cite{dirkse1995path}, are widely used to solve MCPs efficiently and are considered state-of-the-art tools for this purpose. The normalized solution remains a computationally efficient and conceptually straightforward approach for solving GNEPs, despite its limitations in representing more complex dynamics.

\subsection{Limitations of Normalized GNE solutions}

As described above, the normalized solution is usually just one of many possible solutions. In some cases, this solution may not provide a reasonable or intuitive solution. To illustrate this, we solve the problem presented in \ref{exm_2cars} using the normalized solution assumption. Assuming that the Lagrange multipliers of the shared constraints are identical for all players, the solution of the GNEP gives the following GNE:

\begin{equation} \label{norm_exp:sol1}
\begin{split}
    x_1 &= x_1(0)+\Delta t^2 =1.0\\
    x_2 = x_3=\frac{1}{2}&(x_2(0) + x_3(0)+\Delta t^2) =1.125
\end{split}
\end{equation}
The solution obtained from the normalized approach is counterintuitive. In the normalized solution, Car 3 does not block Car 2 as expected. Instead, it moves forward, worsening its own cost function. If Car 3 had remained stationary, it would have achieved a lower cost. The normalized GNE constrained by the fact that $\sigma$ is shared by all players gives a solution the car 3 would not have chosen under any other reasonable condition. It can be easily verified that any solution of the form $x_2=x_3 \in [0.75,1.125]$ is a GNE. Therefore, the normalized solution is merely one possible option for Car 3, and it is not the most appealing for team Car 1 and Car 2.


\section{Non-Normalized Solution For Dynamic Games} \label{sec:Non_Normalized}
The non-normalized solution of a GNEP is a solution $x \in \mathcal{X}$ such that the KKT conditions hold and the Lagrange multipliers of the shared constraints are not the same for all players. As mentioned in Section~\ref{sec:Normalized}, the most popular tools to solve a GNEP are QVI or MCP reformulations, which produce the normalized solution when used as discussed in the previous literature. While these formulations offer many mathematical and numerical advantages, including compatibility with widely available solvers, generalizing these formulations to allow for non-normalized solutions can be beneficial in many cases as shown by the example in section \ref{exm_2cars}. Next, we propose a simple and novel method to exploit the MCP reformulation of a GNEP to calculate non-normalized solutions. Recall the KKT conditions from the original GNEP formulation (\ref{KKT}). The main idea was to compute a single set of shared Lagrange multipliers and introduce a strictly positive scaling factor for each player's shared constraints. Now let's define for every player $i$ a diagonal matrix of weights:
\begin{equation} \label{factors}
\begin{aligned}
    &~~~~~~~~~~~~~~A_i \in \mathcal{D}^+_m, \\
    \mathcal{D}^+_m := \{&D \in \mathbb{R}^{m \times m} \mid D = \text{diag}(d_1, \cdots, d_m), \\ 
    &d_i \in \mathbb{R}_{++} \, \forall i = 1, \cdots, m \}.
\end{aligned}
\end{equation}

Each $A_i$ is used to scale the fictitious shared Lagrange multipliers for player $i$ to produce the actual Lagrange multipliers. Incorporating these factor matrices into the Lagrangian of each player $i$ gives:
\begin{align}
    \mathcal{L}_i := J_i(x^i,x^{-i})+ \mu_i^Th_i(x^i) &+ \lambda_i^T \cdot g_i(x^i) + \\ &+(A_i \sigma)^T s(x^i,x^{-i})
\end{align}
where $\sigma \in \mathbb{R}^{m_0}$ is the fictitious Lagrangian multipliers vector which is shared by all players.

The KKT conditions in (\ref{KKT}) are reformulated as:
\begin{equation} \label{modifiedKKT}
\begin{aligned}
    & \nabla_{x^i}\mathcal{L}_i(x^i, x^{-i},A_i)= 0, \quad \forall i \in \{1, \dots, M\}, \\
    & 0 \leq \lambda_i \perp g_i \leq 0, \quad \forall i \in \{1, \dots, M\}, \\
    & 0 \leq \sigma \perp s \leq 0, \\
    & \lambda_i \geq 0,h_i=0, \quad \forall i \in \{1, \dots, M\}, \\
    & \sigma \geq 0
\end{aligned}
\end{equation}

In this formulation, a single set of shared Lagrange multipliers $\sigma$ is computed, with the scaling applied differently for each player using the factor matrices $A_i$.  The following theorem shows that a solution to the KKT conditions in (\ref{modifiedKKT}) is a GNE solution.

\begin{theorem} \label{thm1}
    Given a set of matrices $\{A_i\}_{i=1}^M$ as defined in (\ref{factors}) and the tuple $(\bar{x}, \{\bar{\mu}_i\}_{i=1}^M, \{\bar{\lambda}_i\}_{i=1}^M, \bar{\sigma})$ that solves the KKT conditions in (\ref{modifiedKKT}), then $\bar{x}$ is a solution to the GNEP if a suitable constraint qualification holds.
\end{theorem}
~
\begin{proof}
The proof of Theorem~\ref{thm1} follows directly from Theorem \ref{thm0}. Define the Lagrange multipliers $\bar{\sigma}_i \in \mathbb{R}^{m_0}$ for player $i$ as:
\begin{equation} \label{eff_sigma}
    \bar{\sigma}_i = A_i \bar{\sigma}, \quad \forall i \in \{1, \dots, M\}.
\end{equation}

The tuple $(\bar{x}, \{\bar{\mu}_i\}_{i=1}^M, \{\bar{\lambda}_i\}_{i=1}^M, \{\bar{\sigma}_i\}_{i=1}^M)$ that solves the KKT conditions in (\ref{KKT})
and thus by Theorem \ref{thm0}, $\bar{x}$ is a solution of the GNEP.
\end{proof}
Theorem~\ref{thm1} provides a method for calculating non-normalized solutions to a GNEP using the tools developed for normalized solutions. In order to compute the  solution  one can use the same MCP formulation as in (\ref{eq:normalized_MCP}) with the same available solvers (e.g., PATH), with the modified Lagrangian that includes the factor matrices. If all factor matrices $A_i$ are identical, the solution reduces to the normalized GNE. 
Without loss of generality, since the factor matrices $A_i$ do not have an absolute scale, we can impose a normalization rule to reduce the number of free parameters. For instance, the factor matrix of player~1 can be fixed as $A_1 = I_{m_0}$. The same solution $(\bar{x}, \{\bar{\mu}_i\}_{i=1}^M, \{\bar{\lambda}_i\}_{i=1}^M, \{\bar{\sigma}_i\}_{i=1}^M)$ remains valid if all factor matrices $A_i$ are divided by a common scalar factor and $\bar{\sigma}$ is multiplied by that same factor, ensuring that the resulting Lagrange multipliers $\bar{\sigma}_i$ remain unchanged. 

An alternative normalization rule can require that, for each shared constraint, the sum of the scaling factors across all players equals one:

\begin{equation} \label{eff_sigma1}
    \sum_{i=1}^N (A_i)_{jj}=1 \qquad\forall j=1,...,m_0.
\end{equation}

Next, we provide two results that offer an interpretation of the matrices $A_i$ and help build intuition for the proposed formulation. 
The first result establishes that, under appropriate scaling (as introduced above) and the assumptions of Rosen’s setting, each solution of the GNEP corresponds to a unique set of factor matrices $A_i$.

\begin{corollary}\label{cor:num1}
 Given a Rosen's Setting GNEP, and if LICQ holds and given a factor matrices normalization rule then every solution of the GNEP corresponds to a unique set of factor matrices (on the set of active shared constraints) and every solution of the GNEP can be expressed by a unique set of factor matrices.
\end{corollary}
\begin{proof}[Proof of Corollary~\ref{cor:num1}]
This corollary follows directly from Theorem~\ref{thm0}(3) and Theorem~\ref{thm1}.
\end{proof}

Corollary~\ref{cor:num1} implies that the entire solution set of the GNEP can be explored by solving for different admissible configurations of the factor matrices $A_i$ that satisfy the predefined normalization rule. 
Note, however, that not every configuration of factor matrices necessarily yields a feasible or valid equilibrium.

Building on this, the next result characterizes the effect of varying the factor matrices on the players’ costs. 
Specifically, Corollary~\ref{cor:num2} shows that, for a class of separable and convex problems with quadratic shared constraints, reducing the scaling factors of a given player relative to the others cannot worsen that player’s cost.



\begin{corollary} \label{cor:num2}
Consider Rosen's setting of a generalized Nash equilibrium problem (GNEP) with a shared quadratic constraint.
Suppose the following hold:
\begin{enumerate}[label=(\roman*)]
    \item Each player's cost is separable: 
    \[
    J_i(x_i,x_{-i}) = \phi_i(x_i), \qquad i=1,\dots,M
    \]
    with no cross-dependence on the other player’s decision variable, and $\phi_i$ are strongly convex.
    \item The shared constraint is quadratic with symmetric block diagonal matrix $Q$, 
    \[
    s(x) =  \frac{1}{2}x^TQx + q^\top x + b=\frac{1}{2}\sum_{i=1}^Mx_i^TQ_{ii}x_i+\sum_{i=1}^Mq_i^Tx_i+b
    \]
    \item A factor-matrix normalization rule is imposed.
    \item $(x^1,\dots,x^M)$ is a solution of the GNEP corresponding to some set of factor matrices $\{A_i\}_{i=1}^M$, and The linear independence constraint qualification (LICQ) holds at the solution.
\end{enumerate}

Then, reducing the values of the factor matrix $A_i$ of some player $i$ relative to the others will \emph{not deteriorate} the cost of that player $i$.
\end{corollary}

\begin{proof}[Proof of Corollary~\ref{cor:num2}]
For simplicity we first show the proof for corollary \ref{cor:num2} for a 2 player dynamic game with a single active shared constraint:

KKT conditions of the problem and assuming that the shared constraint is active:
{\small
\begin{equation} \label{col0}
\begin{aligned}
F_1(x_1,x_2,\sigma,\alpha) &:= \nabla_{x_1} J_1(x_1,x_2) + \alpha\sigma\, \nabla_{x_1} s(x_1,x_2)=0,\\
F_2(x_1,x_2,\sigma,\alpha) &:= \nabla_{x_2} J_2(x_1,x_2) + (1-\alpha)\sigma\, \nabla_{x_2} s(x_1,x_2)=0,\\
F_3(x_1,x_2) &:= s(x_1,x_2)=0.
\end{aligned}
\end{equation}
}
Differentiate $F(x,\sigma,\alpha)=0$ w.r.t.\ $\alpha$ and solve the bordered linear system:
\begin{equation}\label{col1}
\underbrace{\begin{bmatrix}
H_{11} & H_{12} & \alpha\,\nabla_{x_1}s \\
H_{21} & H_{22} & (1-\alpha)\,\nabla_{x_2}s \\
(\nabla_{x_1}s)^{\!\top} & (\nabla_{x_2}s)^{\!\top} & 0
\end{bmatrix}}_{\mathcal{J}}
\begin{bmatrix}
\dfrac{d x_1}{d\alpha}\\[4pt]
\dfrac{d x_2}{d\alpha}\\[4pt]
\dfrac{d \sigma}{d\alpha}
\end{bmatrix}
=
\begin{bmatrix}
-\sigma\,\nabla_{x_1}s\\[4pt]
\phantom{-}\sigma\,\nabla_{x_2}s\\[4pt]
0
\end{bmatrix}
\end{equation}
where
{\small
\begin{align*}
H_{11}&=\nabla^2_{x_1x_1} \phi_1 + \alpha\sigma\, Q_{11}, &
H_{12}&=\alpha\sigma\, Q_{12},\\
H_{21}&=(1-\alpha)\sigma\, Q_{21}, &
H_{22}&=\nabla^2_{x_2x_2} \phi_2 + (1-\alpha)\sigma\, Q_{22} \\
\nabla_{x_1}s &= Q_{11}x_1+q_1, & \nabla_{x_2}s &= Q_{22}x_2+q_2
\end{align*}
}

Define compact blocks:
\begin{equation} \label{col7}
    \begin{aligned}
    H&=\begin{bmatrix} H_{11} & H_{12}\\[2pt] H_{21} & H_{22}\end{bmatrix},\qquad
    g=\begin{bmatrix} \alpha\,\nabla_{x_1}s \\ (1-\alpha)\,\nabla_{x_2}s \end{bmatrix},\qquad \\
    r&=\begin{bmatrix} \nabla_{x_1}s \\ \nabla_{x_2}s \end{bmatrix},\qquad\qquad
    b=\sigma\begin{bmatrix} \nabla_{x_1}s \\ -\,\nabla_{x_2}s \end{bmatrix}.
    \end{aligned}
\end{equation}

Solve for $\dfrac{d\sigma}{d\alpha}$ and $\dfrac{dx}{d\alpha}=[\dfrac{dx_1}{d\alpha},\dfrac{dx_2}{d\alpha}]^T$  with the Schur complement. Notice that the matrix $H \succ 0$  - strong convexity of $\phi_i$ and the PSD of $\sigma Q_{ii}$ means that $H^{-1}$ exists and is PD:

\begin{equation}\label{col4}
\frac{d\sigma}{d\alpha} 
= \,\frac{r^{\!\top} H^{-1} b}{\,r^{\!\top} H^{-1} g\,}, 
\qquad
\frac{d x}{d\alpha} 
= -\,H^{-1}\!\left(g \,\frac{d\sigma}{d\alpha} + b\right)
\end{equation}
Now,
\[
\begin{aligned}
\frac{dJ_1}{d\alpha}
= \nabla_{x_1}J_1^{\!\top}\frac{dx_1}{d\alpha}
  + &\nabla_{x_2}J_1^{\!\top} \frac{dx_2}{d\alpha}
= \\ &= \underbrace{\begin{bmatrix}\nabla_{x_1}J_1 \quad \nabla_{x_2}J_1\end{bmatrix}}_{c^T}
\frac{dx}{d\alpha}
= c^{\!\top}\frac{dx}{d\alpha}.
\end{aligned}
\]

Using $F_1=0$ and separability, we have $\nabla_{x_2}J_1=0$ and
\[
\nabla_{x_1}J_1 \;=\; -\,\alpha\sigma\,\nabla_{x_1}s.
\]
Hence
\[ 
c \;=\; \begin{bmatrix}\nabla_{x_1}J_1 \\[2pt] 0\end{bmatrix}
\;=\; -\,\alpha\sigma\,r_1,
\quad
r_1 := \begin{bmatrix}\nabla_{x_1}s\\[2pt]0\end{bmatrix}, ~
r_2 := \begin{bmatrix}0\\[2pt]\nabla_{x_2}s\end{bmatrix}.
\]
Recall from \eqref{col7} that
\[
\begin{aligned}
r &= \begin{bmatrix}\nabla_{x_1}s\\ \nabla_{x_2}s\end{bmatrix} = r_1 + r_2, \\
g &= \begin{bmatrix}\alpha\,\nabla_{x_1}s\\ (1-\alpha)\,\nabla_{x_2}s\end{bmatrix}
= \alpha r_1 + (1-\alpha) r_2, \\
b &= \sigma\begin{bmatrix}\nabla_{x_1}s\\ -\,\nabla_{x_2}s\end{bmatrix}
= \sigma(r_1 - r_2).
\end{aligned}
\]

\noindent
Introduce the $H^{-1}$–inner product $\langle u,v\rangle= \langle u,v\rangle_{H^{-1}} := u^\top H^{-1}v$ and set
\[
a := \langle r_1,r_1\rangle > 0,\quad
d := \langle r_2,r_2\rangle > 0,\quad
c_0 := \langle r_1,r_2\rangle=0,
\]
so that $ad \ge 0$. A direct calculation gives
\[
r^\top H^{-1} b = \sigma(a-d),
\qquad
r^\top H^{-1} g = \alpha a + (1-\alpha) d.
\]
Therefore, from \eqref{col4},
\[
\frac{d\sigma}{d\alpha}
= -\,\sigma\,\frac{a-d}{\alpha a + (1-\alpha)d}.
\]

\noindent
Using $\dfrac{dx}{d\alpha} = -\,H^{-1}\!\big(g\,\dfrac{d\sigma}{d\alpha} + b\big)$ and $c=-\alpha\sigma r_1$,
\[
\frac{dJ_1}{d\alpha}
= c^\top \frac{dx}{d\alpha}
= -\,c^\top H^{-1}\!\left(g\,\frac{d\sigma}{d\alpha} + b\right)
= \alpha\sigma\, r_1^\top H^{-1}\!\left(g\,\frac{d\sigma}{d\alpha} + b\right).
\]
Substituting the identities above and simplifying yields the compact expression

Thus $J_1$ is non-decreasing in $\alpha$, so reducing the factor of player~1 (i.e., decreasing $\alpha$ under the normalization rule) cannot worsen player~1's cost. This completes the proof for the two-player case with a single active shared constraint.

The same reasoning can be extended to a general $M$-player game with multiple shared constraints under the given assumptions. Specifically, we define

\begin{equation}\label{eq:ext1}
\begin{aligned}
r_i &:= [0,\dots,\nabla_{x_i}s,\dots,0] \in \mathbb{R}^n, \\
H   &:= \operatorname{blkdiag}\big(
       \nabla^2_{x_1x_1}\phi_1 + \alpha_1\sigma Q_{11}, \dots, \\
&\qquad \qquad \qquad \dots,\nabla^2_{x_Mx_M}\phi_M + \alpha_M\sigma Q_{MM}
       \big) \succ 0 .
\end{aligned}
\end{equation}

Stacking the KKT stationarity and shared-constraint equations and 
differentiating with respect to $\alpha_i$ yields the same bordered linear system 
as in the two-player case. Solving it via the Schur complement gives
\[
\frac{d\sigma}{d\alpha_i}
= 
-\,\sigma\,\frac{\langle r,\,r_i\rangle_{H^{-1}}}{\langle r,\,g\rangle_{H^{-1}}},
\qquad
\frac{dx}{d\alpha_i}
= 
-\,H^{-1}\!\left(
g\,\frac{d\sigma}{d\alpha_i} + b_i
\right),
\]
where 
$r := \sum_{k=1}^M r_k$, 
$g := \sum_{k=1}^M \alpha_k r_k$, 
and 
$b_i := \sigma \bigl(r_i - \sum_{k\neq i} r_k\bigr)$. 
Because the costs are separable, 
$\nabla_x J_i = [0,\dots,\nabla_{x_i}J_i,\dots, 0] 
= -\alpha_i\sigma\,r_i$, 
so
\begin{equation}\label{eq:dJi_dalpha}
\begin{aligned}
\frac{dJ_i}{d\alpha_i}
&= 
\nabla_x J_i^\top \frac{dx}{d\alpha_i} \\[6pt]
&=
\alpha_i \sigma^2
\,
\frac{
\langle r_i,r_i\rangle_{H^{-1}}\,\langle r,r\rangle_{H^{-1}}
- 
\langle r_i,r\rangle_{H^{-1}}^2
}{
\langle r, g\rangle_{H^{-1}}
}
\;\ge 0 .
\end{aligned}
\end{equation}

by the Cauchy--Schwarz inequality in the $H^{-1}$ inner product, and under the block-diagonal assumption $Q_{ij} = 0$ for $i\neq j$, 
we obtain 
$\langle r_i,r_j\rangle_{H^{-1}} = 0$ for $i\neq j$, 
the denominator simplifies to 
$\langle r,g\rangle_{H^{-1}}=\sum_{k=1}^M \alpha_k |r_k|_{H^{-1}}^2>0$,
Hence, $\tfrac{dJ_i}{d\alpha_i}\ge 0$, establishing that decreasing the entries of $A_i$ relative to those of the other players cannot worsen player~$i$'s cost in the general $M$-player case. This completes the proof.
\end{proof}

The factor matrices $A_i$ can be interpreted as representing the relative ``aggressiveness" of each player: lower values in $A_i$ for active shared constraints indicate a reduction in the corresponding Lagrange multipliers, allowing the player to improve their cost relative to others. Notice that the diagonal entries of $A_i$ are strictly positive, and therefore any positive terms in $\bar{\sigma}$ (representing active constraints) remain positive in $\bar{\sigma}_i$, while zero terms in $\bar{\sigma}$ remain zero in $\bar{\sigma}_i$.


\section{Analytical Examples} \label{sec:examples}
We use two analytical examples to show the effectiveness of the method presented. 
\subsection{Example 1} 

The first is an example used in \cite{facchinei2007generalized,facchinei2009generalized}. The problem is a dynamic game between 2 player with the following setup (the variables notation are as in the original papers):
\begin{equation} \label{exm1:setup}
\begin{alignedat}{2}
    & \min_x \, (x-1)^2 \qquad \qquad && \min_y \, \left(y - \frac{1}{2}\right)^2 \\
    & \text{s.t.} \quad && \text{s.t.} \\
    & x + y \leq 1 \qquad && x + y \leq 1
\end{alignedat}
\end{equation}
Where $x$ and $y$ are the strategy of each player. Applying the technique presented in this paper, we write the Lagrangian for each player with $\alpha_1$ and $\alpha_2$ as the factors on the shared constraints Lagrangian of each player (since there is a single shared constraint, the factors are scalars):
\begin{equation} \label{exm1:Lagra}
\begin{split}
    &\mathcal{L}_1 = (x-1)^2 + \alpha_1\sigma (x+y-1) \\
    &\mathcal{L}_2 = (y-\frac{1}{2})^2 + \alpha_2\sigma (x+y-1)
\end{split}
\end{equation}

Applying the KKT conditions and according to Theorem \ref{thm1} for every given $\alpha_1$ and $\alpha_2$, the solutions for $x^*$ and $y^*$ are NE:
\begin{equation} \label{exm1:final_sol}
    x^*=1-\frac{\alpha_1}{2(\alpha_1+\alpha_2)} \qquad y^*=\frac{1}{2}-\frac{\alpha_2}{2(\alpha_1+\alpha_2)}
\end{equation}
Furthermore, we can define $\alpha$ such that:
\begin{equation} \label{exm1:final_sol2}
    \alpha = \frac{\alpha_1}{\alpha_1+\alpha_2} \Rightarrow 1-\alpha=\frac{\alpha_2}{\alpha_1+\alpha_2} \Rightarrow \left\{
\begin{aligned}
    & x^* = 1-\frac{1}{2}\alpha \\
    & y^* = \frac{\alpha}{2} \\
\end{aligned}
\right.
\end{equation}
Since $\alpha_1>0$ and $\alpha_2>0$ it suggests that $\alpha \in (0,1)$. The solution for $\alpha\rightarrow0$ gives the solution that converges to $(x^*=1,y^*=0)$ which is the unconstrained solution for player 1 and player 2 just responds - Stackelberg Equilibrium where player 1 is the leader and player 2 is the follower. The solution for $\alpha\rightarrow1$ gives the solution the converges to $(x^*=\frac{1}{2},y^*=\frac{1}{2})$ which is the unconstrained solution for player 2 and player 1 just responds - Stackelberg Equilibrium where player 2 is the leader and player 1 is the follower. For $\alpha=\frac{1}{2}$ the solution corresponds to the normalized solution $(x^*=\frac{3}{4},y^*=\frac{1}{4})$. The proposed method allows to compute all possible GNEs, whose existence was mentioned in \cite{facchinei2007generalized,facchinei2009generalized}.
The second example given here is a dynamic game presented in (\ref{exm_2cars}). Applying the technique presented in this paper, notating the factor as $\alpha_1,\alpha_2,\alpha_3$ for the shared constraints factors, and applying the KKT condition gives the following solution:
\begin{equation} \label{norm_exp:sol4}
\begin{split}
    x_1 &= x_1(0) + \Delta t^2\\
    x_2 = x_3=&\frac{\alpha_3x_2(0)  + \alpha_2x_3(0)+\alpha_3\Delta t^2}{\alpha_2+\alpha_3}
\end{split}
\end{equation}
It can be seen that the solution is independent of $\alpha_1$. Furthermore, we can define $\alpha$ such that:
\begin{equation} \label{exm1:final_sol3}
\begin{split}
    &\alpha = \frac{\alpha_2}{\alpha_2+\alpha_3} \Rightarrow 1-\alpha=\frac{\alpha_3}{\alpha_2+\alpha_3} \Rightarrow \\ &\Rightarrow \left\{
\begin{aligned}
    & v^*_1=1, x^*_1 = 1 \\
    &v_2^* = 1-0.75\alpha \\
    &v_3^* = 0.75(1-\alpha) \\
    & x^*_2 = x^*_3 = \frac{3}{2}-\frac{3}{4}\alpha
\end{aligned}
\right.
\end{split}
\end{equation}
Since $\alpha_2>0$ and $\alpha_3>0$, it suggests that $\alpha \in (0,1)$. This gives all the solutions for $x_2=x_3 \in (0.75,1.125)$. From the equilibrium conditions, it follows that these are all GNEs of the given GNEP.

\subsection{Example 2} 

Another example presented here is Harker's problem solved in \cite{harker1991generalized, nabetani2011parametrized}:
We consider the following two-player Generalized Nash Equilibrium Problem (GNEP):

\begin{align*}
\textbf{Player 1:}\quad 
&\min_{x_1}\; x_1^2 + \frac{8}{3}x_1 x_2 - 34x_1 \\
&\text{s.t.}\; 0 \le x_1 \le 10 \\
&x_1 + x_2 \le 15, \\
\textbf{Player 2:}\quad 
&\min_{x_2}\; x_2^2 + \tfrac{5}{4} x_1 x_2 - 24.25 x_2 \\
&\text{s.t.}\; 0 \le x_2 \le 10 \\ 
&x_1 + x_2 \le 15.
\end{align*}

Our proposed solution procedure involves defining the Lagrangian for each player, introducing a scaling parameter for the shared constraint:

\begin{equation}
\begin{aligned}
    L_1 &= x_1^2 + \frac{8}{3}x_1 x_2 - 34x_1 - \lambda_1^1x_1 + \\ &\qquad\qquad + \lambda_1^2(x_1-10) + \alpha_1\sigma(x_1+x_2-15), \\
    L_2 &= x_2^2 + \tfrac{5}{4} x_1 x_2 - 24.25 x_2 - \lambda_2^1x_2 +  \\&\qquad\qquad + \lambda_2^2(x_2-10) + \alpha_2\sigma(x_1+x_2-15),
\end{aligned}
\end{equation}

where $\sigma$ is the unscaled shared Lagrange multiplier, and $\alpha_1$, $\alpha_2$ are the scaling factors for players 1 and 2, respectively. Without loss of generality, we set $\alpha_1 = 1$, yielding:

\begin{equation}
\begin{aligned}
    L_1 &= x_1^2 + \frac{8}{3}x_1 x_2 - 34x_1 - \lambda_1^1x_1 + \\ &\qquad\qquad \lambda_1^2(x_1-10) + \sigma(x_1+x_2-15), \\
    L_2 &= x_2^2 + \tfrac{5}{4} x_1 x_2 - 24.25 x_2 - \lambda_2^1x_2 + \\& \qquad\qquad \lambda_2^2(x_2-10) + \alpha\sigma(x_1+x_2-15),
\end{aligned}
\end{equation}

where $\alpha = \alpha_2 \in (0, \infty]$.
The KKT conditions for the problem are:

\begin{equation}
\begin{aligned}
    &\frac{\partial L_1}{\partial x_1} = 0, \quad \frac{\partial L_2}{\partial x_2} = 0, \\
    &\lambda_1^1 \cdot x_1 = 0, \quad \lambda_1^2(x_1 - 10) = 0, \\
    &\lambda_2^1 \cdot x_2 = 0, \quad \lambda_2^2(x_2 - 10) = 0, \\
    &\sigma \cdot (x_1 + x_2 - 15) = 0, \\
    &\lambda_1^1, \lambda_1^2, \lambda_2^1, \lambda_2^2, \sigma \ge 0.
\end{aligned}
\end{equation}

The stationary conditions are:

\begin{equation}
\begin{aligned}
    \frac{\partial L_1}{\partial x_1} &= 2x_1 + \frac{8}{3}x_2 - 34 - \lambda_1^1 + \lambda_1^2 + \sigma = 0, \\
    \frac{\partial L_2}{\partial x_2} &= 2x_2 + \frac{5}{4}x_1 - 24.25 - \lambda_2^1 + \lambda_2^2 + \alpha\sigma = 0.
\end{aligned}
\end{equation}

One possible solution arises when all constraints are inactive (i.e., all Lagrange multipliers are zero):

\begin{equation}
\begin{aligned}
    &2x_1 + \frac{8}{3}x_2 - 34 = 0, \\
    &2x_2 + \frac{5}{4}x_1 - 24.25 = 0,
\end{aligned}
\quad \Rightarrow \quad
x_1^* = 5, \quad x_2^* = 9.
\end{equation}

This yields one valid solution. Other solutions arise when at least one constraint is active. For example, if the shared constraint is active ($\sigma > 0$), then:

\begin{equation}
x_1 + x_2 = 15.
\end{equation}

Substituting into the stationarity conditions:

\begin{equation}
\begin{aligned}
    2x_1 + \frac{8}{3}x_2 - 34 + \sigma &= 0, \\
    2x_2 + \frac{5}{4}x_1 - 24.25 + \alpha\sigma &= 0,
\end{aligned}
\end{equation}

Solving this system yields:

\begin{equation}
\begin{aligned}
    \sigma^*(\alpha) &= \frac{8}{8\alpha - 9}, \\
    x_1^*(\alpha) &= \frac{72\alpha - 69}{8\alpha - 9}, \\
    x_2^*(\alpha) &= \frac{48\alpha - 66}{8\alpha - 9}.
\end{aligned}
\end{equation}

Since $\sigma > 0$, we require $\alpha \in \left(\frac{9}{8}, \infty\right]$. Imposing the constraint $x_1 \le 10$ further restricts $\alpha \in \left(\frac{21}{8}, \infty\right]$.

This example demonstrates the analytical tractability and flexibility of the proposed method. By introducing scaling parameters on shared constraints, we can derive closed-form expressions for the entire set of GNE solutions, even in nonlinear and non-convex settings. Compared to the parametrized VI approach, our method offers clearer guarantees of solution validity, reduced dimensionality, and broader compatibility with existing solvers. We believe this framework opens new avenues for efficiently solving complex multi-agent optimization problems.

\section{Selecting a GNE Point Among Many} \label{sec:selection}
In previous section we have shown how to utilize the KKT condition either by directly solving them or by using an MCP formulation to calculate a set of possible GNE of a GNEP.  
To choose a solution from the possible GNE set, a bi-level optimization scheme is suggested in this paper. The optimization problem we propose chooses the set of factor matrices $\{A_i\}_{i=1}^N$ which minimize a given cost function:
\begin{equation} \label{opt:biopt}
\begin{split}
    &\min_{A_i,..,A_N\in \mathcal{D}^+_m} {J_0(x^*)} \\
    & s.t. \quad (x^i)^* \in \mathcal{S}_i(x^{-i},A_1,...,A_M) \quad \forall i \in \{1,...,M\}
\end{split}
\end{equation}
In general, the bi-level optimization (\ref{opt:biopt}) can be hard to solve. The cost to minimize can be of any form and can try to achieve different goals. A few possible objectives include: 1) minimize the sum of all the players' costs. 2) Optimize for the cost of one of the players - this will lead to a strategy close to a Stackelberg Equilibrium (in a 2 players scenario). 3) Optimize over some property of the game, for instance maximize the average velocity in a racing game or maximize the interaction between players (to have a game which is fun to watch). To see the effect of such bi-level optimization, the problems presented in section \ref{sec:examples} are solved by minimizing the sum of all the players' costs:
\begin{equation} \label{opt:cost}
    J_0 = \sum_{i=1}^M{J_i(x^i,x^{-i})}
\end{equation}
In the case of the first example, the bi-level optimization becomes a one-layer optimization since the GNE solution of the DG is known (\ref{exm1:final_sol2}):
\begin{equation} \label{opt:opt1}
\begin{split}
    \min_{0\leq\alpha\leq1} &{(x^*-1)^2+(y^*-\frac{1}{2})^2} \\
    s.t.~&x^*=1-\frac{1}{2}\alpha,\quad  y^*=\frac{\alpha}{2}
\end{split}
\end{equation}
Substituting the optimal solution for both players into the cost:
\begin{equation} \label{opt:opt2}
\begin{split}
    &J_0(\alpha)=\frac{1}{4}(2\alpha^2-\alpha+1) 
    \Rightarrow \frac{dJ_0}{d\alpha}=\alpha-1
    \end{split}
\end{equation}
By setting $\frac{dJ_0}{d\alpha}=0$, the optimal solution given the chosen $J_0$ is $\alpha^*=1$.
Note that, in general, one should exclude boundary solution of  $\alpha$ (i.e. $0$ and $1$) as joint constraints might not be satisfied by one player. In this example $\alpha=1$ yields $x = y = \frac{1}{2}$ which is a feasible solution.
Recall that the normalized solution approach gives a different solution $(x=\frac{3}{4},y=\frac{1}{4})$. For the same type of high level cost function,  the example presented in~(\ref{exm_2cars}) results in the following optimization:
\begin{equation} \label{opt:opt4}
\begin{split}
    \min_{0<\alpha<1} &{-x^*_1+x^*_2+\frac{(v_1^*)^2+(v_2^*)^2+(v_3^*)^2}{2}} \\
    & s.t. \text{ solution in~(\ref{exm1:final_sol3})}
\end{split}
\end{equation}
Substituting the terms for $(x_1^*,v_1^*,x_2^*,v_2^*)$ into $J_0$ 
and solving for the optimal $\alpha^*$ gives 
$\alpha^*=1$ and this corresponds to the solution:
\begin{equation}
    (x_1^*,x_2^*,x_3^*) = (1,0.75,0.75)
\end{equation}
This solution is much more intuitive relative to the normalized solutions. In this case, car 3 actually helps car 1 win the race. This example illustrates a possible problem of considering only a single predefined solution of the GNEP, like the normalized solution.

\section{Two-Cars Racing Problem} \label{sec:results}
One of the most popular approaches to solving an autonomous car racing problem is by using a Dynamic Game (DG) approach. In a 2 cars race, each car has a unique dynamic model and a cost function to minimize. A DG formulation is set up for each car with some shared prediction horizon. The game’s shared constraint is Collision Avoidance (CA) between both players. The DG is solved by finding the GNE of the problem, usually using a MCP formulation solved with a numerical solver, and the shared constraints Lagrange multipliers are assumed to be the same \cite{zhu2023sequential,liu2023learning}. This problem often has multiple GNEs, and there is no clear reason to prefer the normalized solution over others.

A kinematic bicycle model for each car is used, with the following states and inputs:
\begin{equation} \label{states}
\begin{aligned}
    & x^i_k = [v, \psi, s, t, X, Y], \quad i=1,2 \quad k = 0,...,N \\
    & u^i_k = [u_a, u_\delta], \quad i=1,2 \quad k = 0,...,N-1
\end{aligned}
\end{equation}
where, $v$ represents the total velocity, $\psi$ is the angular state relative to the center-line, and $s, v_t$ are the arc-length position of the car in the local Frenet reference frame relative to the center line of the track and lateral deviation from the center-line, respectively. $X$ and $Y$ are the inertial position of the car. $u_a$ is the longitudinal acceleration and $u_\delta$ is the steering of the front wheels $N=10$ is the prediction horizon. Finally, the time step is set to $\Delta t=0.1[sec]$. 

To illustrate the effectiveness of the proposed method, the same dynamic model and cost function is used for both players. The Dynamic Game formulation for player $i \in \{1,2\}$ is given by:
\begin{equation}
\begin{aligned}
    &\min_{x^i, u^i} -s^i_N + s^{-i}_N + \frac{\beta}{2} \sum_{k=0}^{N-1} \|u^i_k\|^2 \\  
    &\text{s.t.} \quad x^i_{k+1} = f(x^i_k, u^i_k), \quad x^i_k \in \mathcal{X}, \quad u^i_k \in \mathcal{U}, \\  
    & (X^1_k - X^2_k)^2 + (Y^1_k - Y^2_k)^2 \geq d_{safe}, \quad k = 1, \dots, N-1,
\end{aligned}
\nonumber
\end{equation}
where, $f$ is the dynamics function, $\mathcal{X}$ is the feasible set of the states, $\mathcal{U}$ is the feasible set of the inputs, and $d_{safe} = 0.4[m]$ is the minimal distance between both vehicles, and $\beta = 10^{-1}$. As can be seen, both cars have the same characteristics and the same cost function properties. Each car aims to maximize its relative progress while minimizing control effort. 


To compute the GNE of this DG, an MCP formulation is numerically solved using the PATH solver \cite{dirkse1995path}. To define the Lagrangian of each car using the method presented in section \ref{sec:Non_Normalized}, a different factor should be given to each collision avoidance (CA) constraint. In this problem, there are $N$ shared constraints for each car. To make the solution more tractable, the number of factors is lowered. The factor matrix of the first car is chosen to be $I_{N}$ as explained in section \ref{sec:Non_Normalized}, and the factor matrix of car 2 is reduced to $\alpha I_{N}$ which can further reduced to a scalar $\alpha$. The physical interpretation of $\alpha \in \mathbb{R}_{++}$ is the relative aggressiveness of player 2 relative to player 1. If $\alpha <1$ player 2 is more aggressive than player 1, and for $\alpha>1$ player 1 is more aggressive than player 2.

This means that there is a single parameter that factors the shared constraints Lagrange multipliers:
\begin{equation}
\begin{aligned}
\label{DG_formulation}
    \mathcal{L}_1& = -s^1_N + s^2_N + \frac{\beta}{2}\sum_{k=0}^{N-1}{\|u^1_k\|^2} +\mu_1^T \cdot h(x^1, u^1) + \\ & \qquad\qquad +\lambda_1^T\cdot g(x^1,u^1)
    +\sigma^T\cdot g_{CA}(x^1,x^2) \\
    \mathcal{L}_2& = -s^2_N + s^1_N + \frac{\beta}{2}\sum_{k=0}^{N-1}{\|u^2_k\|^2} + \mu_2^T \cdot h(x^2,u^2) + \\ & \qquad\qquad + \lambda_2^T\cdot g(x^2,u^2) + \alpha \cdot\sigma^T\cdot g_{CA}(x^1,x^2)
\end{aligned}
\end{equation}
where  $\mu_1, \mu_2$ are the Lagrange multipliers associated with the dynamics (equality) constraints  $h$. $\lambda_1, \lambda_2$ are the Lagrange multipliers associated with the state and input constraints $g$. $\sigma$ is the fictitious Lagrange multiplier vector associated with the CA constraint $g_{CA}$. The next step is to apply the stationary condition on each Lagrangian and complementary slackness conditions:
\begin{equation}
\begin{aligned}
&\frac{\partial \mathcal{L}_1}{\partial(x^1,u^1)} = 0,\quad \frac{\partial \mathcal{L}_2}{\partial(x^2,u^2)} = 0 \\
&x^1_{k+1} - f_1(x^1_k, u^1_k) = 0,\quad x^2_{k+1} - f_2(x^2_k, u^2_k) = 0 \\
& g_1(x^1,u^1) \leq 0,\quad  g_2(x^2,u^2) \leq 0 \\
& g_{CA}(x^1,x^2) \leq 0
\end{aligned}
\end{equation}
Solving the above KKT equations for any value of $\alpha \in (0,\infty)$ gives a valid GNE solution.

\subsection{Quantitative Analysis of Tuning Aggressiveness}

The solution set of two examples are presented here to illustrate the effectiveness and added value of our method. Each example corresponds to a specific initial state of the cars on different tracks.

The first example is calculated on a straight track. Car 1 (blue) is ahead of car 2 (green) but has a velocity disadvantage. Figure \ref{fig:race_straight} shows the solution of the race for different $\alpha$. The solution for $\alpha \rightarrow 0$ gives a straight trajectory (red line) for car 2 and this corresponds to car 2 being more aggressive than car 1. The solution for $\alpha \rightarrow \infty$ gives a straight trajectory (cyan line) for car 1 and this corresponds to car 1 being more aggressive than car 2. The GNE solution changes continuously as $\alpha$ changes. The lower graph shows the cost function of both cars as function of $\alpha$. A car can improve its cost by playing more aggressively. On the left side of the graph the velocity and steering profile of the cars are presented. The black line represent the normalized solution.

\begin{figure}
    \centering
    \includegraphics[scale=0.29]{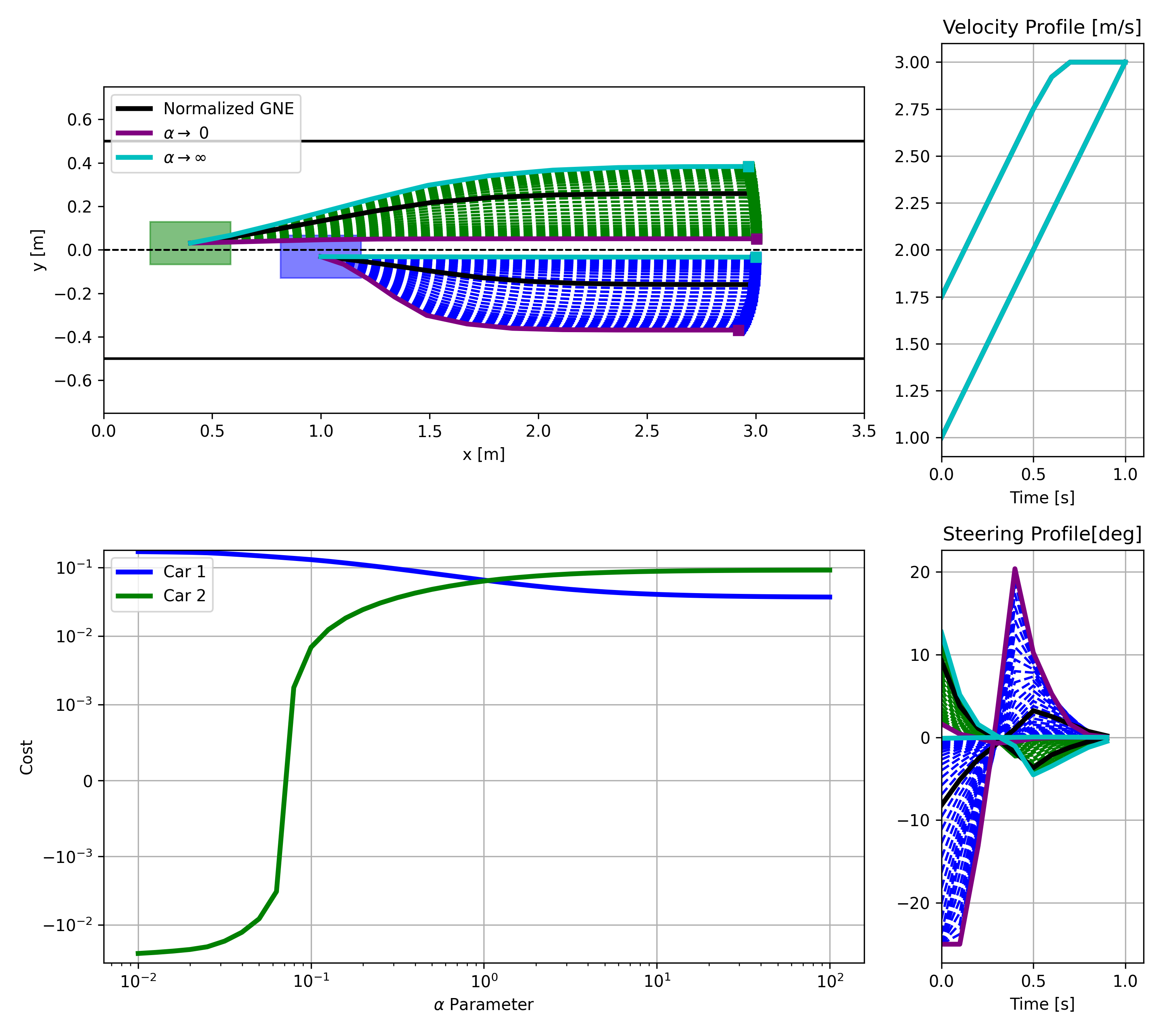}

    \caption{First Qualitative Example: Racing Problem Solutions on a Straight Track}
    \label{fig:race_straight}
\end{figure}

The second example is calculated on a curved track. Car 1 is ahead of car 2 but has a velocity disadvantage. Figure \ref{fig:race_circle} shows the solution of the race for different $\alpha$. In this example there are 2 types of possibles solutions. For $\alpha \rightarrow 0$ Car 2 is much more aggressive and cuts into the corner which enforces Car 1 to take the wider turn. On the other hand, for $\alpha \rightarrow \infty$ car 1 is much more aggressive and it cuts into the corner before car 2 can do it - this enforces car 2 to take the wider turn.

The GNE solution does not change continuously as $\alpha$ changes. The lower graph shows the cost function of both cars as function of $\alpha$. This multi modal solution was not possible if only the normalized solution is calculated. It can be seen that for $\alpha \approx10$ there is a jump in the costs which corresponds to the jump in the solution types. On the left side of the graph the velocity and steering profile of the cars are presented.
\begin{figure}
    \centering
    \includegraphics[scale=0.29]{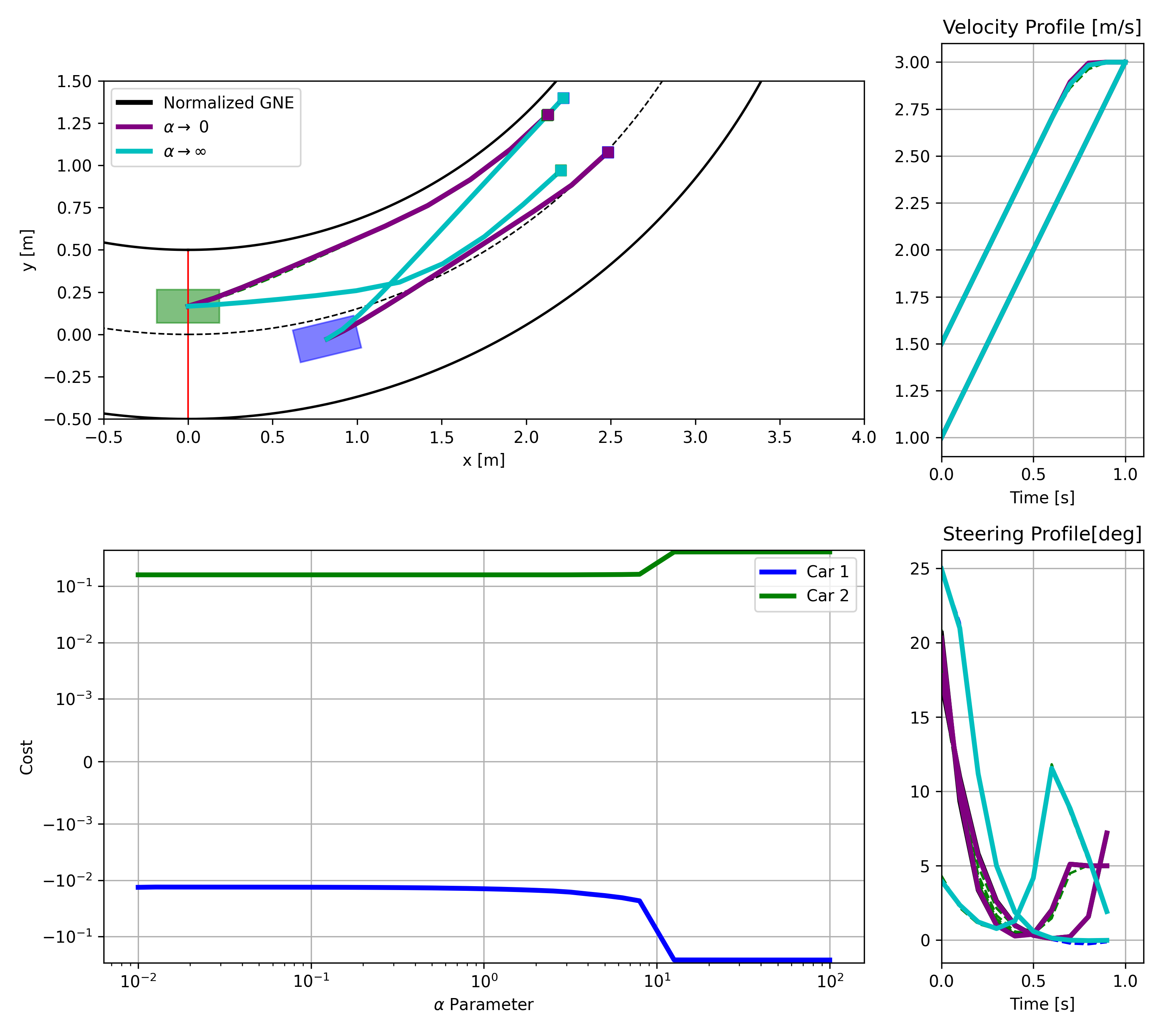}

    \caption{Second Qualitative Example: Racing Problem Solutions on a Curved Track}
    \label{fig:race_circle}
\end{figure}

\subsection{Quantitative Results: A Numerical Study}

The examples above show the set of options a race car has at different situations of the race. Choosing the right strategy from all possible is essential to get an advantage. The intuition suggests that the more aggressive a car will act (meaning, choosing lower $\alpha$) the better its performance will be. To prove this, we present a numerical study of a racing simulation of 2 racing cars. Both cars race by solving a racing DG problem as presented in (\ref{DG_formulation}). 

The scenario investigated is on an L-shaped track where the cars race in the counter-clock direction. The initial state of the cars is randomly initialized along the track with velocities such that the trailing car (called ego) has a velocity advantage over the forward car (Opponent).

Two scenarios are investigated. In the first, both cars implement a strategy that is the result of the normalized solution ($\alpha = 1$). In the second scenario, the opponent continues to implement the normalized strategy (meaning, it solve the DG while assuming both car use a normalized solution) while the ego car solves the DG with non-normalized solution ($\alpha=0.05$).

In both cases, the goal of the ego car is to overtake the opponent during a $2[sec]$ race started from the randomized initial conditions. The simulation is \emph{closed-loop}, meaning that each player solves the DG at each time step, implements a single control input and solves the DG again to get the next time step.

\begin{figure}
    \centering
    \includegraphics[scale=0.14]{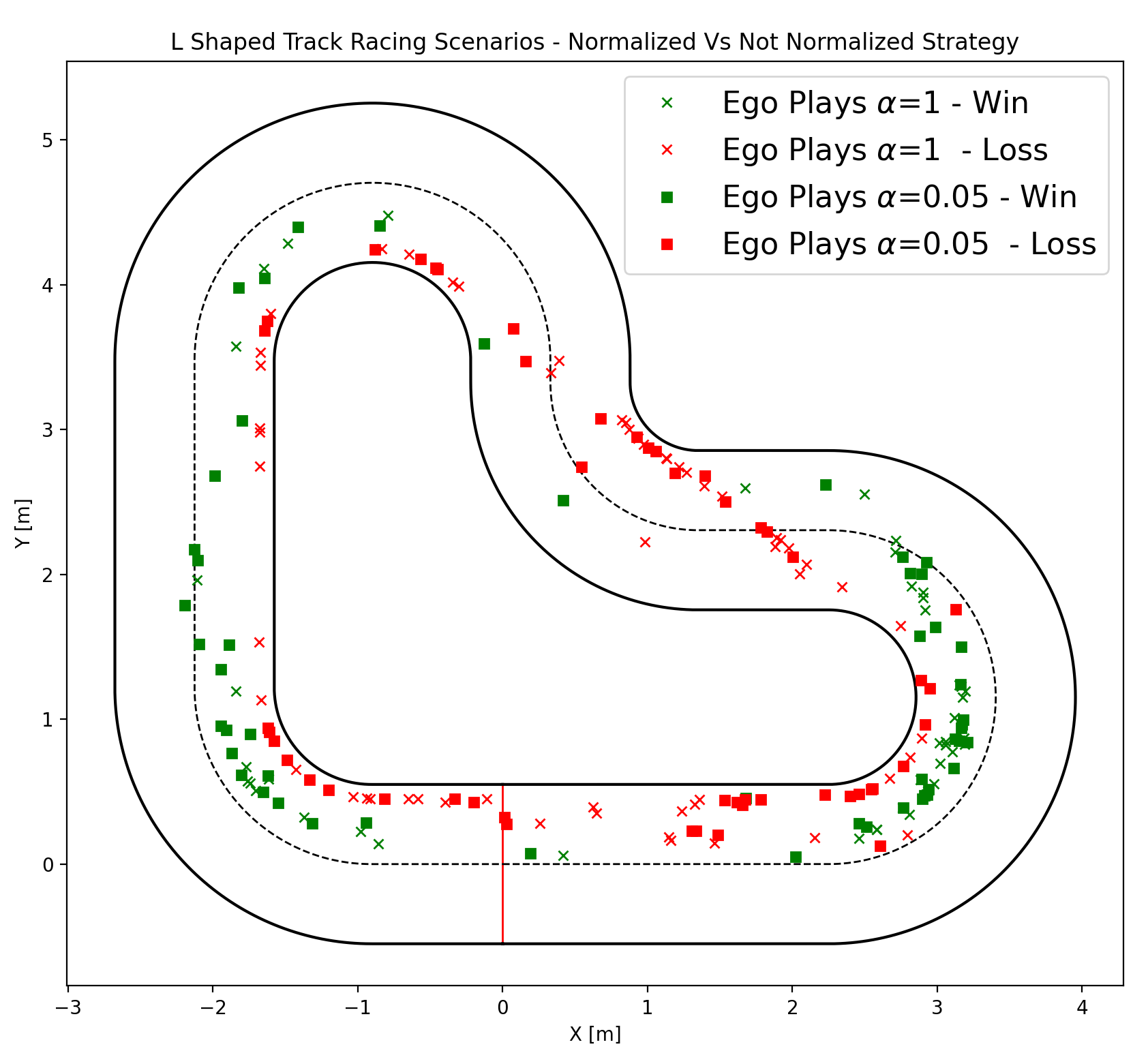}
    \caption{Results of 100 Monte Carlo simulations comparing two different strategies on an L-shaped track}
    \label{fig:MC_res}
\end{figure}

{\small
\begin{table}[h]
    \centering
    \begin{tabular}{|c|c|c|}
        \hline
         & Ego Plays Normalized & Ego Plays Non-Normalized \\ 
        \hline
        Win Percentage & 43\% & 54\% \\ 
        \hline
    \end{tabular}
    \caption{MC simulation results summary comparing both strategies}
    \label{tab:results}
\end{table}
}
The ego has a top speed of $3[m/s]$ while the opponent has a top speed of $2.85[m/s]$, otherwise both cars have the same dynamic model and constraints. The initial state of the cars is randomized with properties presented in the following table:
\begin{table}[h]
    \centering
    \begin{tabular}{|c|c|}
        \hline
        \textbf{Description}  & \textbf{Range} \\ 
        \hline
        Opponent initial position  & $[0, L]$ \\ 
        \hline
        Ego initial position relative to opponent  & $[-1.75, -1.5]$ \\ 
        \hline
        Opponent initial velocity & $[1.0, 2.0]$ m/s \\ 
        \hline
        Ego initial velocity relative to opponent & $[0.25, 0.75]$ m/s \\ 
        \hline
        Ego lateral offset & $[-H/3, H/3]$ m \\ 
        \hline
        Opponent lateral offset (relative to ego) & $[- H/8,H/8]$ m \\ 
        \hline
    \end{tabular}
    \caption{Randomized Parameters in the Monte Carlo Simulation. $L$ is the track length and $H$ is the track half width.}
    \label{tab:random_params}
\end{table}

Table \ref{tab:results} summarizes the results of this study. When both cars use the normalized strategy (exhibiting similar levels of aggressiveness), the ego car overtakes the opponent in $43\%$ of the scenarios. However, when the ego car adopts a more aggressive strategy than the opponent, it wins $54\%$ of the scenarios—an $11\%$ \emph{increase} in win probability. Note that the advantage of the non-normalized strategy may be even greater, considering that certain randomized initial conditions may inherently prevent the ego car from winning, regardless of its strategy.
. 
Another visualization of these results is available in Figure \ref{fig:MC_res}. Each point on the track represents the initial condition of the ego car in a single racing scenario. A green scenario indicates a win for the ego car (successful overtaking of the opponent), while a red scenario indicates a loss (failure to overtake the opponent). This shows the effectiveness of the non-normalized strategy relative to the normalized one.

Overall, these results clearly demonstrate that a more aggressive strategy increases the likelihood of winning. While this conclusion is intuitive, proving it rigorously is nontrivial. The non-normalized solutions presented in this paper provide a systematic method for controlling the relative aggressiveness of players in a game and analyzing its impact on the final outcome.

Also, while the increase in the win percentage from 43\% to 54\% may not seem substantial, the theoretical upper bound on the optimal win rate in this setting is difficult to compute. We conjecture that, given the constraints of this scenario, the optimal win percentage when both players pursue non-normalized solutions is likely not significantly higher than 54\%.

Clearly, in real-time applications, an identification of the opponent's aggressiveness needs to be carried out and adjusted dynamically to optimize individual performance while avoiding crashes. Understanding the effects of time-varying opponent aggressiveness is an interesting direction for future research but is beyond the scope of this paper.

\section{Conclusion} \label{sec:conclusion}
This paper introduces a novel method for extending existing approaches used to solve dynamic games (DG) through Mixed Complementarity Problem (MCP) tools. This traditional method has been restricted to computing the normalized solution. The proposed framework allows for the computation of non-normalized solutions, significantly expanding the solution space in DGs while using the same numerical tools. By leveraging bi-level optimization, the method enables solution selection based on specific objectives, such as minimizing the aggregate cost, optimizing the performance of individual players, or prioritizing game-specific properties. Furthermore, it was shown that by interpreting the different GNE solution by the level of relative aggressiveness of each player, a player can get an advantage by playing more aggressively. The examples demonstrate the effectiveness of the proposed framework in practical applications, highlighting its ability to identify diverse GNE solutions and to tailor these solutions to specific objectives. This approach addresses the limitations of normalized solutions by accounting for terms that are neglected in the standard solution. In conclusion, this work expands the computational toolbox for DGs by enabling exploration beyond normalized solutions and by offering a structured methodology for equilibrium selection. These contributions pave the way for more robust and flexible applications of game-theoretic principles in complex, real-world scenarios.


\bibliographystyle{IEEEtran}
\bibliography{IEEEfull,main}

\begin{IEEEbiography}[{\includegraphics[width=1.1in,height=1.4in,clip,keepaspectratio]{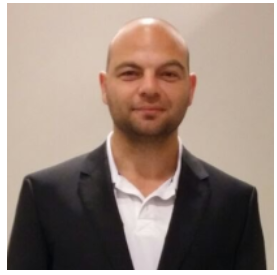}}]{Mark Pustilnik}
Mark Pustilnik received his B.S.c in 2007 and M.Sc. in 2013 in Aerospace Engineering from the Technion – Israel Institute of Technology. 
He is currently a Ph.D. candidate in Mechanical Engineering at the University of California, Berkeley. 
His research focuses on dynamic games, generalized Nash equilibria, optimal control, and trajectory optimization.
\end{IEEEbiography}

\begin{IEEEbiography}[{\includegraphics[width=1.1in,height=1.4in,clip,keepaspectratio]{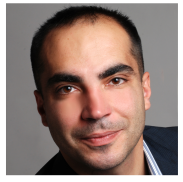}}]{Prof. Francesco Borrelli}
Francesco Borrelli received the Laurea degree in
computer science engineering from the University of
Naples Federico II, Naples, Italy, in 1998, and the
Ph.D. degree from ETH-Zurich, Zurich, Switzerland,
in 2002.
He is currently an Associate Professor with the
Department of Mechanical Engineering, University
of California, Berkeley, CA, USA. He is the author
of more than 100 publications in the field of pre-
dictive control and author of the book Constrained
Optimal Control of Linear and Hybrid Systems
(Springer-Verlag). His research interests include constrained optimal control,
model predictive control and its application to advanced automotive control
and energy efficient building operation.
Dr. Borrelli received the 2009 National Science Foundation CAREER
Award and the 2012 IEEE Control System Technology Award. In 2008, he
became the Chair of the IEEE Technical Committee on Automotive Control.
\end{IEEEbiography}

\end{document}